\documentclass[12pt]{article}

\usepackage{amsmath,amssymb,amsthm,mathtools}
\usepackage{setspace,enumitem,soul,url,xcolor,hyperref,bm,pifont,cancel,comment,booktabs}
\usepackage{tikz-cd}

\newtheorem{proposition}{Proposition}
\newtheorem{coro}{Corollary}
\newtheorem{lemma}{Lemma}
\newtheorem{example}{Example}

\usepackage[T1]{fontenc}
\usepackage{tgbonum} \fontfamily{qbk} 

\title{On design-unbiased algorithmic \\ Machine Learning}
\date{}
\author{\normalsize Li-Chun Zhang\footnote{Correspondence: Department of Social Statistics and Demography, Univ. of Southampton, Highfield SO17 1BJ, Southampton, UK. Email: L.Zhang@soton.ac.uk}, University of Southampton \& Statistisk sentralbyrå \\
\normalsize Siu-Ming Tam, Australian Bureau of Statistics \\
\normalsize Luis Sanguiao-Sande, Instituto Nacional de Estadística \\
\normalsize Wesley Yung, Statistics Canada (Ret.) \\
\normalsize Anders Holmberg, Australian Bureau of Statistics
}

\begin{document}

\maketitle

\begin{abstract}
Machine Learning (ML) algorithms, such as k-Nearest Neighbours (kNN) or random forest, eschew the ideal of true data models in favour of predictive performance. However, minimising the MSE or F-score cannot lead to unbiasedness directly, which is important in many situations such as official statistics. We study the conditions of algorithmic ML, other than the existence and knowledge of true data models, which lead to unbiased prediction or classification for a given finite population, including how the training data may be sampled from the population, how a trained prediction algorithm can be tuned to achieve unbiased prediction or classification for that population, and how the performance of out-of-sample prediction or classification can be assessed unbiasedly. The inference is based on the \emph{known} probability design of samples and training sets, rather than any \emph{assumed} distributions or models. 
\end{abstract}


\section{Introduction}

Breiman (2001a) contrasts algorithmic and data modelling cultures, where the ideal of true data models is eschewed in favour of predictive performance of ML algorithms,  such as kNN or random forest. However, minimising the MSE or F-score cannot directly lead to unbiasedness of prediction or classification, which is important in many finite-population applications of ML. 

For instance, many official statistics are defined as descriptive summaries of a country's population, economy, society or environment. While the quality of official statistics is multi-dimensional (e.g. United Nations, 2019; European Commission, 2017; Statistics Canada, 2017), accuracy and reliability as measured in terms of bias and variance, respectively, are central in all the quality frameworks necessary to maintain the public trust. 
 
As a common approach to producing official statistics, survey sampling shares Breiman's perspective on ML algorithms. The focus is to improve the sampling strategy, consisting of a probability sampling design and an associated choice of estimator (Neyman, 1934). The inference is with respect to the known sampling design, regardless of the existence or knowledge of a true data model. See Hansen (1987), Smith (1994), Kalton (2002), Rao (2005, 2011), Beaumont and Haziza (2022) for reviews and appraisals.

The well-known Horvitz-Thompson estimator (Horvitz and Thompson, 1952) is the most typical example, which is exactly unbiased over repeated sampling from a given population. To improve the efficiency by leveraging auxiliary information about the population, i.e. in addition to the sampling design, it is popular to adopt a design-based model-assisted approach (e.g. Särndal et al., 1992), whereby a prediction model relating the target outcome to the known covariates can be introduced to adjust the Horvitz-Thompson estimator, but the inference of bias and variance is still with respect to the sampling design, without the need for the assisting model to be the true data model. 

Breidt and Opsomer (2017) summarise a general ``recipe'' of model-assisted estimation, which adjusts the model-prediction of the population total by the weighted sum of sample prediction residuals, where the weights are inverse sample inclusion probabilities as in the Horvitz-Thompson estimator. 

In case the assisting model is pre-trained, i.e. not learned on the actual sample, the resulting difference estimator (e.g. Särndal et al., 1992, Sec. 6.3) is exactly design-unbiased. See also Angelopoulos et al. (2023) for a related approach to independent and identically distributed (IID) samples, either by stipulation or by with-replacement sampling from finite populations.
 
Otherwise, and more common in practice, estimating the model means that the resulting estimator is no longer exactly unbiased. To justify any given ML algorithm in this context, some authors resort to the notation of consistency asymptotically, as the population and sample sizes increase to infinity, under suitable regularity conditions. See McConville and Toth (2019) for regression trees generated by the recursive partitioning algorithm, or Dagdoug et al. (2023) for random forest by the original algorithm of Breiman (2001b).  

However, as Smith (1994) points out, the ``asymptotic notion of consistency'' is not immediately applicable to the given population as ``a real entity''. One may observe an alternative notion of consistency following the works of Fisher (1956) and Neyman (1934). For a given population and sampling method, if $\hat{\theta}$ is unbiased for the vector of population totals $\theta$, then $g(\hat{\theta})$ is called ``consistent'' for $g(\theta)$ by Fisher (1956); whereas an interval estimator of a population statistic is called ``consistent'' by Neyman (1934), if it achieves the designated level of coverage. Zhang et al. (2025) refer to such finite-sample design-unbiased estimators as \emph{Neyman-Fisher consistent}.

Sanguiao-Sande and Zhang (2021) propose an approach to Neyman-Fisher consistent population total estimation, which is exactly design-unbiased in the finite-sample setting. In particular, given any ML algorithm as the assisting model, one can apply Rao-Blackwellisation (Rao, 1945; Blackwell, 1947) to the Horvitz-Thompson estimator of the total prediction errors in the population outside the training set to achieve design-unbiasedness. The resulting estimator still uses the sampling design weights explicitly, like all other traditional design-based model-assisted estimators.

Zhang et al. (2025) extend the \emph{subsampling Rao-Blackwell (SRB)} technique of Sanguiao-Sande and Zhang (2021) to a larger class of estimators, called the \emph{prediction estimators} (Royall, 1970; Valliant et al., 2000), where a prediction estimator of a population total is the sum of the observed sample total and a predicted out-of-sample total. Since one can plug in any ML algorithm for the latter, the prediction estimator can be constructed by a working model of the population without using the sampling weights at all, although doing so would often cause bias over repeated sampling from the given population. Applying Rao-Blackwellisation to the subsample-trained prediction estimator, Zhang et al. (2025) obtain exactly design-unbiased estimators of the bias and MSE of the resulting SRB-prediction estimators.  

However, Zhang et al. (2025) did not study the general conditions of ML that can lead to design-unbiased prediction estimators. Nor did they consider the estimation of classification accuracy when the prediction estimator is given by unit-level classifications of the out-of-sample units. 

In this paper, we shall focus on the conditions of ML under finite population sampling, which would yield design-unbiased estimation of population totals, either by predicting or classifying the out-of-sample units. Since the inference is design-based, the unbiasedness holds regardless of the `truthfulness' of the assisting model or ML algorithm. This is especially useful in applications of ML, where unbiasedness for finite populations is of critical importance. 

The specific questions we consider are: (i) how the training data may be sampled from the population, so as to enable the inference of out-of-sample prediction errors based on the observed in-sample test errors, (ii) how an ML algorithm, which is formed by the training data, can be tuned by the test errors to yield design-unbiased out-of-sample prediction, and (iii) how out-of-sample unit-level classification can yield design-unbiased prediction estimation, and how to assess the associated classification accuracy unbiasedly.    

In the rest of the paper, Sections \ref{sec:prediction} and \ref{sec:classification} deal with unbiased prediction and classification, respectively. Sections \ref{sec:kNN} and \ref{sec:application} present illustrations, simulations and an application of the theory developed, using kNN as the example ML algorithm throughout. Section \ref{sec:final} gives concluding remarks.

\section{Design-unbiased prediction} \label{sec:prediction}

Denote by $U = \{ 1, ..., N\}$ a given finite population, with known feature vector $x_i$ for each $i\in U$. Denote by $s$ the sample of $n$ units with observed \emph{outcome} $y_i$ for each $i\in s$, whereas $y_i$ is unknown for the rest of the population. We treat all the values $\{ (y_i, x_i) : i \in U\}$ as constants, whether they are known or not. The variation of ML and prediction is due to the following two elements. 

First, let the sample $s$ be selected from $U$ by a probability sampling design, to be referred to as the \emph{$p$-design} and denoted by 
\[
s \sim p(s) 
\]
where $p(s)$ sums to 1 over all possible $s$, and $\pi_i = \Pr(i\in s) >0$ for any $i\in U$. Next, for algorithmic ML, let $s_1$ be the \emph{training} set taken from $s$, and let $s_2 = s\setminus s_1$ be the corresponding \emph{test} set, which are created according to a specific design, to be referred to as the \emph{$q$-design} and denoted by
\[
s_1 \sim q(s_1 \mid s). 
\]
For instance, $s_1$ of size $n_1$ may be selected from $s$ by simple random sampling without replacement --- simply SRS from now on. Or, $s_1$ may be created by bootstrap of $s$, in which case a given unit in $s$ may be selected multiple times in $s_1$, and $s_2$ contains the units in $s$ not selected to $s_1$ at all. Or, by $L$-folding, one would create $L$ test sets $s_2$ by a random $L$-partition of $s$, and the $L$ training sets $s_1 = s\setminus s_2$ are determined accordingly. 

In any case, we consider the $pq$-design to be well-defined for inference, if the joint distribution of $(s_1, s)$ admits the factorisation
\begin{equation} \label{pq}
q(s_1\mid s) p(s) = f(s_1) f(s\mid s_1)
\end{equation}
such that the non-empty test set $s_2$ can be regarded as a probability sample from $U\setminus s_1$ conditional on $s_1$. Notice that the trivial $q$-design, $\Pr(s_1 = s \mid s) = 1$, does not yield a well-defined $pq$-design since $s_2 \equiv \emptyset$, although it is feasible for ML with default setups, such as kNN with pre-fixed $k$. Well-defined $pq$-designs, however, are necessary for the inference of ML. 
 
As a typical target of interest, let us consider the population $y$-total, which can be decomposed into the observed sample total and the unknown total in the rest of the population, denoted by
\[
Y = \sum_{i\in U} y_i = \sum_{i\in s} y_i + \sum_{i\notin s} y_i = \sum_{i\in s} y_i + Y_R
\] 
where the subscript denotes $R =U\setminus s$. Let a \emph{prediction estimator} of $Y$ be
\[
\hat{Y} = \sum_{i\in s} y_i + \sum_{i\notin s} \hat{y}_i = \sum_{i\in s} y_i + \hat{Y}_R
\]
where $\hat{y}_i$ can be given by any ML prediction or classification algorithm.

In particular, as the out-of-sample total $Y_R$ varies with $s$ under repeated sampling $s\sim p(s)$, unbiased estimation of $Y$ (as a constant) is equivalent to unbiased prediction of $Y_R$ (as a random variable), denoted by
\[
E_p(\hat{Y}) = Y \quad\Leftrightarrow\quad E_p(\hat{Y}_R - Y_R) = 0~.
\]

\subsection{Representative training}

Denote by $\mu(x, s_1)$ a predictor obtained from the training data $\{ (y_j, x_j : j\in s_1\}$, which is aimed at the outcome $y$-value given the associated feature vector $x$. The notation signifies that, for any unit outside the training set, $i\notin s_1$ with $x_i = x$, its predicted $y$-value $\mu(x, s_1)$ varies only with $x$ and $s_1$. 

\paragraph{Definition} A well-defined $pq$-design is said to yield \emph{representative training} of $\mu(x, s_1)$ if, $\forall i\in U$, we have
\begin{equation} \label{RT}
E_{pq}\big( \mu(x_i, s_1) \mid i\in s_2\big) = E_{pq}\big( \mu(x_i, s_1) \mid i\notin s\big) .
\end{equation}

In other words, representative training is the case, by the given $pq$-design, if the expected predictor $\mu(x_i, s_1)$ is the same for any unit $i\in U$ conditional on its being outside the training set, regardless whether the unit needs to be predicted (i.e. $i\notin s$) or is observed and may be used for inference (i.e. $i\in s_2$). For simplicity, we shall say a unit $i$ is \emph{out-of-bag (OOB)} if it is in the sample but outside the training set, i.e. $i\in s_2$, whereas it is out-of-sample if $i\notin s$.

The concept of representative training \eqref{RT} is intuitive, since it allows one to connect the unobserved out-of-sample performance of $\mu(x, s_1)$ to its observed in-sample OOB performance. For instance, conditional on $s_1$, all the prediction errors $\mu(x_i, s_1) - y_i$ can be partitioned according to $U\setminus s_1 = s_2 \cup R$, where $s_2$ is a sample from $U\setminus s_1$ with respect to the $pq$-design, which would allow us to make inference of the prediction bias of $\hat{Y}_R = \sum_{i\in R} \mu(x_i, s_1)$ later.

\begin{lemma} \label{RT:lemma}
A well-defined $pq$-design yields representative training, 
\begin{itemize}[leftmargin=6mm,itemsep=0pt] \setstretch{1.0}
\item for all possible ML algorithms $\mu(x,s_1)$ if and only if, $\forall i\in U$,  we have
\begin{equation} \label{RT:all}
\pi_{2i} \coloneq \Pr(i\in  s_2 \mid s_1) = \frac{\pi_i - \Pr(i\in s_1)}{1 - \Pr(i\in s_1)} \mathbb{I}(i\notin s_1)
\end{equation}
\item for any given ML algorithm $\mu(x,s_1)$ if and only if, $\forall i\in U$,  we have
\begin{equation} \label{RT:given}
\text{Cov}_{s_1}\big( \mu(x_i, s_1), \pi_{2 i} \mid i \notin s_1 \big) = 0 .
\end{equation}
\end{itemize}
\end{lemma}

The proof is given in Appendix \ref{app:proof}, as are the proofs of other results in the paper. Since condition \eqref{RT:all} implies \eqref{RT:given}, the latter may be regarded as a more general condition for representative training: it means that \emph{any} well-defined $pq$-design may yield representative training of certain models, i.e. those satisfying \eqref{RT:given}. However, one may be unable to verify the condition \eqref{RT:given}, which requires all the relevant unconditional and conditional sampling probabilities to be known. The condition \eqref{RT:all} is thus more readily applicable, given which the $pq$-design guarantees representative training of \emph{all} possible ML algorithms.  

\begin{coro} \label{SRS-SRS} 
Given $s$ by SRS from $U$ and $s_1$ by SRS from $s$, the SRS-SRS $pq$-design satisfies \eqref{RT:all}. 
\end{coro}

\begin{example}
Let $U = \{ i_1, i_2, i_3, i_4\}$, and $(n, n_1) = (2, 1)$ for the SRS-SRS $pq$-design. We have $\pi_{2i} = \frac{n-n_1}{N-n_1} = \frac{1}{3}$, $\forall i\in U$, which satisfies \eqref{RT:all} since $\pi_i = \frac{2}{4}$ and $\Pr(i\in s_1) = \frac{1}{4}$.  
\begin{itemize}[leftmargin=6mm]
\item Given $i_1\in s_2$, we have $s = \{ i_1, i_2\}, \{ i_1, i_3\}, \{ i_1, i_4\}$, yielding 3 distinct training sets $s_1 = \{ i_2\}, \{ i_3\}, \{ i_4\}$ without $i_1$, each with probability $f(s_1 \mid i_1\in s_2) = \frac{1}{3}$.
\item Given $i_1\notin s$, we have $s = \{ i_2, i_3\}, \{ i_2, i_4\}, \{ i_3, i_4\}$, yielding 3 distinct training sets $s_1 = \{ i_2\}, \{ i_3\}, \{ i_4\}$, each with probability $f(s_1 \mid i_1\notin s) = \frac{1}{3}$. Notice that the same training set can occur in different $s$ without $i_1$, such as $s_1 = \{ i_2\}$ in $s = \{ i_2, i_3\}$ or $s = \{ i_2, i_4\}$, but the probability of $s_1 = \{ i_2\}$ is still $\frac{2}{6} = \frac{1}{3}$.
\end{itemize} 
It is clear that, given any $\mu$ trainable on $s_1$, the expectation of $\mu(x_i, s_1)$ is the same, either conditional on $i\in s_2$ or $i\notin s$, for any $i\in U$, i.e. representative training. 
\end{example}

Apart from SRS, it is common to generate $(s_1, s_2)$ by $L$-folding or bootstrap sampling. For $L$-folding, suppose (i) $n_2 = n/L$ is an integer given the sample size $n$ of $s$, and (ii) the $L$ sets of $s_2$ are drawn sequentially from $s$ by SRS. We would then have the same $f(s_1)$ by SRS-$L$-folding as by the SRS-SRS $pq$-design, and representative training of all possible $\mu(x, s_1)$; whereas, without (i) and (ii), representative training can be the case approximately.  

Next, when $(s_1, s_2)$ are given by bootstrap sampling of $s$, the probability $f(s_1)$ would depend on the number of distinct units in $s_1$, denoted by $m$, as well as their realised frequencies in $s_1$. The expected frequency is the same for each distinct unit, and it is well known that $m/n$ tends to $1- e^{-1}$ as $n$ increases. By the symmetry of this SRS-bootstrap $pq$-design, $\pi_{2i} = \Pr(i\in s_2 \mid s_1)$ is approximately a constant for $i\notin s_1$, given sufficiently large $n$, and we have approximately representative training for all possible  $\mu(x, s_1)$. 

\begin{coro} \label{Pois-Ber}
Let $s$ be given by Poisson sampling from $U$, and $s_1$ by Bernoulli sampling from $s$. The Poisson-Bernoulli $pq$-design satisfies \eqref{RT:all}. 
\end{coro}

Corollary \ref{Pois-Ber} suggests that, for an unequal-probability $p$-design with fixed sample size $n$, a negligible sampling fraction $n/N$ can yield approximately representative training as long as one adopts an equal-probability $q$-design.

\subsection{Prediction unbiasedness and tuning}

The set of out-of-sample units $R = U\setminus s$ varies with $s\sim p(s)$. The predictor $\mu(x, s_1)$ obtained from $s_1$ is $pq$-unbiased for out-of-sample prediction, if 
\begin{equation} \label{unbias-pq}
E_{pq}\Big( \sum_{i\in R} \mu(x_i, s_1) \Big) = E_p\Big( \sum_{i\in R} y_i \Big)
\end{equation}  
over all the out-of-sample units. We leave it to future studies to investigate if design-unbiased prediction of unit-specific $y_i$ can be achieved sensibly. 

Below we provide a sufficient condition for \eqref{unbias-pq} given representative training by $pq$-design. Consequently, we show how it can be used to \emph{tune} the given $\mu(x, s_1)$ to achieve unbiased prediction.

\begin{proposition} \label{OOB}
Given representative training of $\mu(x, s_1)$ by the $pq$-design, it is $pq$-unbiased for out-of-sample prediction \eqref{unbias-pq} if, for any $s\sim p(s)$, we have 
\begin{equation} \label{OOB-q}
\sum_{i\in s} \big(\tfrac{1}{\pi_i} -1\big) E_q\big( \mu(x_i, s_1) \mid i \in s_2\big) =  \sum_{i\in s} \big( \tfrac{1}{\pi_i} -1 \big) y_i .
\end{equation}
\end{proposition}

\begin{coro} Under the SRS-SRS $pq$-design, the predictor $\mu(x, s_1)$ is $pq$-unbiased for out-of-sample prediction if, for any $s\sim p(s)$, we have
\begin{equation} \label{SRS-q}
\sum_{i\in s} E_q\big( \mu(x_i, s_1) \mid i \in s_2\big) = \sum_{i\in s} y_i .
\end{equation}
\end{coro}

The \emph{$q$-benchmarking} condition \eqref{OOB-q} is stated for the given sample $s$, without any unobserved $y$-values. It is likely that the equality will not hold exactly in practice, unless the predictor is rather simplistic, such as training-set mean given an SRS-SRS $pq$-design. Although this does not imply that the predictor $\mu(x, s_1)$ must be biased, since the condition is sufficient not necessary, a non-negligible discrepancy between the two sides of \eqref{OOB-q} may be a reason for tuning to achieve out-of-sample prediction unbiasedness. 

Meanwhile, let the \emph{subsampling Rao-Blackwell (SRB) predictor} of any out-of-sample unit with $x_i = x$ be given as
\begin{equation} \label{SRB}
\bar{\mu}(x, s) = E_q\big( \mu(x, s_1) \mid s \big) .
\end{equation} 
By definition, $\bar{\mu}(x, s)$ is more efficient than $\mu(x, s_1)$ because it removes the extra training variance $V_q\big( \mu(x, s_1)\mid s\big)$, and it has the same out-of-sample prediction bias as $\mu(x, s_1)$ which can be removed as follows.

\begin{proposition} \label{prop:OOB-tuned}
Given a representative training $pq$-design, with $\pi_i = \Pr(i\in s)$ and $q_{2i} \coloneq \Pr(i\in s_2\mid i\in s)$, where $n_R$ is the number of out-of-sample units, the predictor $\tilde{\mu}(x, s)$ is $pq$-unbiased for out-of-sample prediction, where
\[
\tilde{\mu}(x, s) = \bar{\mu}(x, s) - \tau_{\mu}(s) 
\]
and
\begin{equation} \label{OOB-tuned}
\tau_{\mu}(s) = \tfrac{1}{n_R} \Big\{ E_q\Big( \sum_{i\in s_2} \big( \tfrac{1}{\pi_i} -1\big) \tfrac{1}{q_{2i}} \mu(x_i, s_1)\Big) 
- \sum_{i\in s} \big( \tfrac{1}{\pi_i} -1\big) y_i \Big\} .
\end{equation}
\end{proposition}

\begin{coro} Given the SRS-SRS $pq$-design, any given SRB-predictor \eqref{SRB} can be tuned by \eqref{OOB-tuned} to be unbiased for out-of-sample prediction \eqref{unbias-pq}, where $\tau_{\mu}(s)$ is the average in-sample OOB prediction error which is given as
\begin{equation} \label{OOB-SRS}
\tau_{\mu}(s) = E_q\Big( \tfrac{1}{n_2} \sum_{i\in s_2} \big( \mu(x_i, s_1) - y_i \big) \Big) 
= \tfrac{1}{n} \sum_{i\in s} E_q\big( \mu(x_i, s_1) \mid i\in s_2\big) - \tfrac{1}{n} \sum_{i\in s} y_i ~.
\end{equation}
\end{coro}

\section{k-Nearest Neighbor (kNN) prediction} \label{sec:kNN}

Let us illustrate with kNN as a simple ML algorithm. Given feature vector $x$ and donor set $s$ of size $n$, let $\mathcal{N}_k(x;s)$ denote the set of $k$ nearest neighbours (or units in $s$) in terms of $x$, such that the kNN predictor can be given as
\[
\eta(x, s) = \frac{1}{k} \sum_{j\in \mathcal{N}_k(x;s)} y_j ~.
\]

According to Stone (1977, Theorem 2 and Corollary 3), the kNN predictor is IID-consistent for $E_{XY}(y \mid x)$ for any fixed $x$-value, as both the sample size $n$ and $k$ tend to $\infty$, where $E_{XY}$ denotes model expectation over the distribution of $(x,y)$. While the IID assumption may not be unreasonable in many applications, it seems unrealistic to require $k\rightarrow \infty$ in finite-sample settings where kNN is typically applied with a small fixed $k$. 

Let us apply design-based prediction theory under the SRS-SRS $pq$-design, where explicit results are available, which can as well be instructive for those interested in IID-model inference. Let the population mean and its generic prediction estimator be given as, respectively, 
\[
P = \frac{1}{N} \sum_{i\in U} y_i \qquad\text{and}\qquad \hat{P} = \frac{1}{N} \Big( \sum_{i\in s} y_i + \sum_{i\notin s} \hat{\mu}_i \Big) 
\quad\text{with}\quad \hat{\mu}_i = \mu(x_i, s)~.
\]

\subsection{Debiasing kNN}

Let us consider the following baseline, or kNN-based, predictors.
\begin{align*} 
\text{Sample-mean (baseline):}\quad & \hat{\mu}_i = \bar{y}_s = \tfrac{1}{n} \sum_{i\in s} y_i \\
\text{Sample-kNN:}\quad & \hat{\mu}_i =  \eta(x_i, s) \\
\text{SRB-kNN}: \quad & \hat{\mu}_i =  \bar{\eta}(x_i, s) = E_q\big( \eta(x_i, s_1) \mid s\big) \\
\text{SRB-kNN, residual-tuned:} \quad & \hat{\mu}_i =  \bar{\eta}(x_i, s) - \Big\{ \tfrac{1}{n} \sum_{k\in s} \bar{\eta}(x_k, s) - \bar{y}_s \Big\} \\
\text{SRB-kNN, OOB-tuned:} \quad & \hat{\mu}_i =  \bar{\eta}(x_i, s) - \Big\{ \tfrac{1}{n} \sum_{k\in s} \dot{\eta}(x_k, s) - \bar{y}_s \Big\}
\end{align*} 
where $\dot{\eta}(x_k, s_1)$ is the OOB SRB-predictor for any $k\in s$, i.e. 
\[
\dot{\eta}(x_k, s) = E_q\big( \eta(x_k, s_1) \mid k\in s_2 \big) ~.
\]

The sample-mean yields $pq$-unbiased prediction of $\bar{Y}_R = \frac{1}{N-n} \sum_{i\notin s} y_i$, but it does not use the $x$-feature information. Both sample-kNN and SRB-kNN are likely to be $pq$-biased: while we are not aware of any exactly unbiased estimator of the bias of sample-kNN, an unbiased bias estimator for SRB-kNN is given by the term in $\{ \cdot \}$ of the OOB-tuned SRB-kNN according to the theory above. The residual-tuned SRB-kNN is not unbiased, although it may have a smaller bias than SRB-kNN, whereas the OOB-tuned SRB-kNN is exactly $pq$-unbiased. Let us illustrate with a simple example, where exact calculation is possible.

\begin{example} \label{ex:toy}
In a population $U$ of size $N$, let $x_i = 1$ if $i=1,2$ and $x_i = 0$ if $i\neq1, 2$. Let $y_1 =1$ and $y=0$ if $i\neq 1$. Suppose an SRS-SRS $pq$-design with $k=1$ for the kNN predictor. By straightforward though tedious calculation (Appendix \ref{app:toy}), we can show that the OOB-tuned SRB-kNN is exactly unbiased, whereas the prediction bias of $\bar{Y}_R$ by the different kNN-predictors are given as
\begin{align*}
\text{Bias(sample-kNN)} & =  \frac{1}{N-1} \Big( \frac{n}{N-1} - 1 \Big) ~,\\
\text{Bias(SRB-kNN)} & =  \frac{1}{N-1} \Big( \frac{n_1}{N-1} - 1 \Big) ~,\\
\text{Bias(residual-tuned SRB-kNN)} & \overset{n_1=n_2}{=} \frac{1}{N-1} \Big( \frac{5n}{4(N-1)} - 1 \Big) ~.
\end{align*}
Notice that the bias of residual-tuned SRB-kNN depends on the choice of $(n_1, n_2)$, and the result given above refers to the case of $n_1 = n_2$.
\end{example}

\subsection{A simulation study} \label{sec:study}

We can use simulations to illustrate the potential bias of using kNN predictors and design-based debiasing. Let a population of values be generated as below: 
\begin{gather*}
y_i = 1 + 0.5 x_{1i} + 0.3 x_{2i} + 6 \max(0, x_{1i})^2 + \epsilon_i ~,\\
x_{1i}, x_{2i} \overset{\text{IID}}{\sim} N(0, 1)~, ~Cov(x_{1i}, x_{2i}) = 0.5 \quad\text{and}\quad 
\epsilon_i \overset{\text{IID}}{\sim} N(0, 0.2) ~.
\end{gather*}
Figure \ref{fig:popSim} depicts such a simulated population of size $N=250$. As can be seen, there is a dramatic `bend' of $E(y_i \mid x_{1i})$ somewhere between $x_{1i} =0$ and $x_{1i} =1$, which can potentially cause bias in kNN-prediction since the realised $y$-values are rather imbalanced on either side of $x_{1i}$ as $x_{1i}$ increases. 

\begin{figure}[ht]
\centering
\includegraphics[width=14cm,height=8cm]{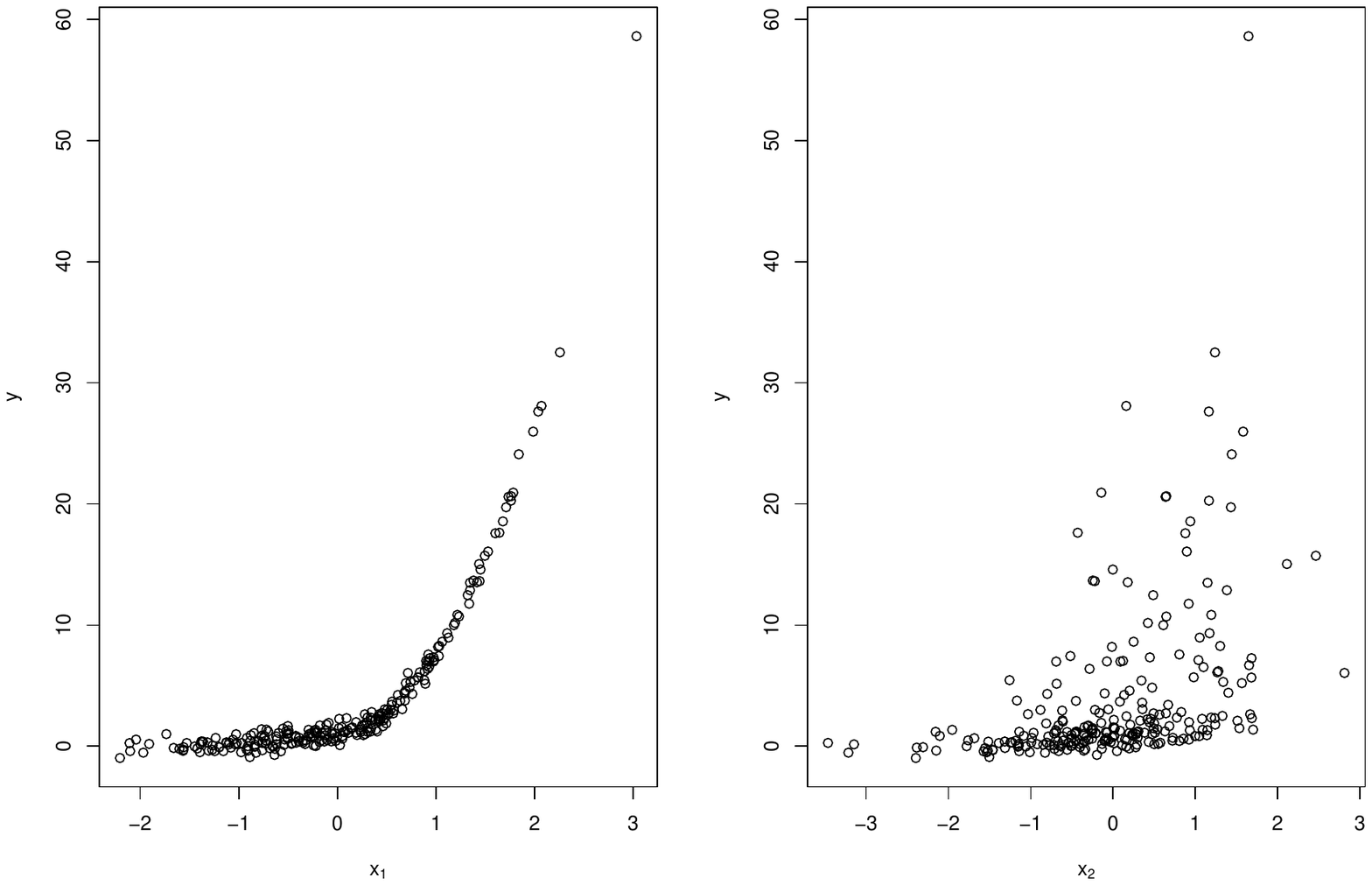}
\caption{Scatter plots of simulated population, $N=250$.}
\label{fig:popSim}
\end{figure}

We now conduct a Monte Carlo simulation with $B$ samples $s$ drawn by SRS from a simulated population, with population size $N=10^4$ and sample size $n=10^3$. Generically, let $\hat{\theta}^{(1)}, ..., \hat{\theta}^{(B)}$ be the realised values of a given population mean prediction estimator over the $B$ simulated samples. Given the population mean $P$, let the \emph{empirical} bias and MSE of $\hat{\theta}$ by simulation be
\[
\text{EBias}(\hat{\theta}) = \frac{1}{B} \sum_{b=1}^B \big( \hat{\theta}^{(b)} -P \big) \qquad\text{and}\qquad 
\text{EMSE}(\hat{\theta}) = \frac{1}{B} \sum_{b=1}^B \big( \hat{\theta}^{(b)} - P \big)^2 .
\]
In addition, let the \emph{empirical} standard error of $\hat{\theta}$ by simulation be
\[
\text{ESE}(\hat{\theta}) = \sqrt{\frac{1}{B-1} \sum_{b=1}^B \big( \hat{\theta}^{(b)} -  \frac{1}{B} \sum_{b=1}^B \hat{\theta}^{(b)} \big)^2} .
\]

In case $\hat{\theta}$ refers to an SRB-kNN prediction estimator, we can estimate its bias and MSE as described in Appendix \ref{app:inference} . Let $\overline{\text{bias}}(\hat{\theta})$ be the average of the $B$ bias estimates and $\overline{\text{mse}}(\hat{\theta})$ that of the MSE estimate. Next, $\overline{\text{mse}}(\hat{\theta})$ is available when $\hat{\theta}$ is the unbiased sample-mean $\bar{y}_s$. Finally, $\overline{\text{bias}}$ and $\overline{\text{mse}}$ of the sample-5NN are omitted in the results below, since we are unaware of any exactly unbiased estimator of the bias or MSE of sample-kNN.

\begin{table}[ht]
\centering
\caption{Results by SRS-SRS $pq$-design, $B=100$.}
\begin{tabular}{lrrrrrr} \toprule
& Sample & \multicolumn{3}{c}{SRB-5NN} & Sample \\ \cline{3-5}
& 5NN & Untuned & Resid-tuned & OOB-tuned & mean \\ \hline
EBias & -0.225 & -0.795 & -0.099 & -0.088 & -0.069 \\
ESE  & 0.028 & 0.300 & 0.118 & 0.122 &  0.211 \\
EMSE & 0.051 & 0.721 & 0.024 & 0.023 & 0.049 \\ 
$\overline{\text{bias}}$ & N/A & -0.696 & -0.011 &  0  & 0 \\
$\overline{\text{mse}}$ & N/A & 0.726 & 0.016 & 0.016 & 0.050 \\ \bottomrule
\end{tabular} \label{tab:sim}
\end{table}

Table \ref{tab:sim} shows the simulation results, based on $B=100$ samples $s$, where kNN-prediction uses $k=5$. For the SRB-kNN predictors, we used $T=100$ for the Monte Carlo SRB and $n_1 = 50$ as the training set size; see Appendix \ref{app:MC} for details. Notice that the subsample donor set size $n_1 = 50$ is rather small compared to the sample size $n=1000$, which is chosen here purely for illustration purposes, because the bias of kNN-prediction tends to increase in magnitude as the donor set size decreases, i.e. $n_1$ for subsample-kNN or $n$ for sample-kNN. 

We note the following about these results. Firstly, the population mean is $P =4.107$ in this case. The bias of the sample-5NN or SRB-5NN prediction estimator is not negligible: EBias is seen to dominate the corresponding ESE  for both, which is $-0.225$ versus $0.028$ for the sample-5NN and $-0.795$ versus $0.300$ for SRB-5NN. 

Secondly, the residual-tuned SRB-kNN is not unbiased, while the OOB-tuned SRB-kNN and the sample-mean are both unbiased. However, comparing the three EBiases by simulation, i.e. $-0.099$ versus $-0.088$ or $-0.069$ in Table \ref{tab:sim}, residual-tuning does seem to have removed most of the bias of kNN in this case. Nevertheless, debiasing by the OOB-tuned SRB-kNN is seen to be beneficial here, now that it achieves unbiasedness as well as the lowest EMSE (i.e. $0.023$) among all the alternatives.

Thirdly, when it comes to bias and MSE estimation, the result is acceptable for the sample-mean $\bar{y}_s$, where $\overline{\text{mse}} =0.050$  compared to $\text{EMSE} = 0.049$. For the SRB-5NN, we have $\overline{\text{bias}} = -0.696$ compared to $\text{EBias} = -0.795$ and $\overline{\text{mse}} =0.726$ to $\text{EMSE} = 0.721$. For the OOB-tuned SRB-5NN that is unbiased, $\overline{\text{mse}} =0.016$ is reasonable compared to its $\text{ESE}^2 = 0.015$. For the residual-tuned SRB-5NN, we have $\overline{\text{mse}} =0.016$, which is clearly lower than its $\text{EMSE} = 0.024$ but close to $\overline{\text{bias}}^2 + \text{ESE}^2 = (-0.011)^2 + 0.118^2 = 0.014$. This again suggests that its EBias here likely overstates its bias due to the Monte Carlo error. 

In summary, the simulation study here has demonstrated the potential bias caused by direct kNN-prediction, such that debiasing by our theory may be beneficial. Moreover, the associated inference of SRB-prediction enables one to choose among the  alternative predictors in practice. Although residual-tuning may be able to remove much of the bias, unbiased OOB-tuning is preferable when its MSE is not larger than that of residual-tuning.

\section{Design-based classification} \label{sec:classification}

We have so far considered design-unbiased prediction. Unit-level classification of categorical outcomes may be of practical interest as well. Below we consider specifically classification of binary outcomes.  

\subsection{Out-of-sample classification accuracy}

Let $\mu(x, s_1)$ be the probability of $y =1$ according to $\mu(\cdot)$ trained on $s_1$, such as the kNN-mean given $x$ and donor set $s_1$. Let $\hat{y}_i(s_1)$ be a binary classifier based on $\mu(x_i, s_1)$ for any $i\in U$. For instance, $\hat{y}_i(s_1) = \mathbb{I}\big( \mu(x_i, s_1) > 0.5\big)$ deterministically, or $\hat{y}_i(s_1) = \mathbb{I}\big(u_i \leq \mu(x_i, s_1)\big)$ randomly given $u_i \overset{\text{IID}}{\sim} \text{Unif}(0,1)$ where $u_i$ is independently generated across the units for the purpose of classification. 

Let the \emph{out-of-sample classification accuracy} of $\hat{y}_i(s_1)$ refer to its performance averaged over all the out-of-sample units $R =U\setminus s$, which is defined as
\begin{equation} \label{cls1}
\phi(R\mid s_1) = \frac{1}{n_R} \Phi(R\mid s_1)\quad\text{and}\quad \Phi(R\mid s_1) = \sum_{i\in R} \mathbb{I}\big(y_i = \hat{y}_i(s_1) \big)
\end{equation}
As with unbiased prediction \eqref{unbias-pq}, we leave it to future study to determine whether design-unbiased classification or  inference of unit-specific $y_i$ can be achieved in a sensible manner. 

It follows that deterministic classification causes biased results generally. For instance, suppose the population proportion $P_x = 0.2$ given $x$ is known, then $\mathbb{I}(P_x \geq \psi)$ is either 0 or 1 for \emph{all} the units with $x_i =x$, given any threshold value $\psi \in (0,1)$, which is clearly biased. Meanwhile, let $N_x$ be the number of units with $x_i = x$ in the population, then the randomised classifier $\mathbb{I}(u_i \leq P_x)$ achieves unbiased classification, in the sense that
\[
E_u\Big( \tfrac{1}{N_x} \sum_{i\in U: x_i = x} \mathbb{I}\big(y_i = \mathbb{I}(u_i \leq P_x) \big) \Big) 
= P_x =  \tfrac{1}{N_x} \sum_{i\in U: x_i = x} y_i
\]
where the expectation is with respect to $u_i \overset{\text{IID}}{\sim} \text{Unif}(0,1)$.

Given $\mu(x, s_1)$ instead of $P_x$ in practice, one can estimate $\phi(R\mid s_1)$ conditional on $s_1$ based on the OOB classification errors $\{ \mathbb{I}\big( y_i = \hat{y}_i(s_1) \big) : i\in s_2\}$,  just like estimating its prediction error based on $\{ \mu(x_i, s_1) - y_i : i\in s_2\}$. However, what is of interest is the classification accuracy associated with the corresponding SRB-predictor $\bar{\mu}(x_i, s)$, which is more efficient than $\mu(x_i, s_1)$.

Let us consider directly the Monte Carlo implementation of SRB based on $T$ random splits, $(s_1^{(t)}, s_2^{(t)})$ for $t=1, ..., T$. This includes the exact SRB-probability $\bar{\mu}(x, s)$ as the special case of $T = \mathcal{C}(n, n_1)$ distinct splits. For any $i\in U$, let
\[
T_i = \sum_{t=1}^T \mathbb{I}(i\notin s_1^{(t)})
\]
be the number of splits where the unit is out of $s_1$, where $T_i = T$ if $i\notin s$. Let 
\begin{equation} \label{OOB-SRB}
\mathring{\mu}(x_i, s) = \frac{1}{T_i} \sum_{t=1}^T \mu(x_i, s_1^{(t)}) \mathbb{I}(i\notin s_1^{(t)})
\end{equation}
for any $i\in U$, which has the same expectation as the exact SRB-probability. It follows that, given a representative training $pq$-design \eqref{RT}, we have  
\begin{equation} \label{unbiasMC}
E_{pq}\big( \mathring{\mu}(x_i, s) \mid i\in s_2\big) = E_p\big( \bar{\mu}(x_i, s) \mid i\notin s_1 \big) 
= E_{pq}\big( \mathring{\mu}(x_i, s) \mid i\notin s\big)
\end{equation}
for any $i\in U$. Let $\mathring{y}_i(s)$ be the classifier using the probability $\mathring{\mu}(x_i, s)$ by \eqref{OOB-SRB}. 
Let its \emph{out-of-sample classification accuracy} be given similarly as \eqref{cls1}, 
\begin{equation} \label{cls}
\mathring{\phi}(R) = \frac{1}{n_R} \mathring{\Phi}(R) \quad\text{and}\quad \mathring{\Phi}(R) = \sum_{i\in R} \mathbb{I}\big(y_i = \mathring{y}_i(s) \big)
\end{equation}

\subsection{OOB estimation of classification accuracy}

Let the unit-specific classification accuracy using the probability $\mathring{\mu}(x_i, s)$ be
\[
\mathring{\alpha}_i = E_{pq}\Big( \mathbb{I}\big( y_i = \mathring{y}_i(s)\big) \mid i\notin s_1 \Big) 
\]
for any $i\in U$.
In other words, for given unit $i$ with fixed $y_i$, $\mathring{\alpha}_i$ measures the performance of $\mathring{y}_i(s)$ over repeated sampling conditional on $i\notin s_1$. Next, let us introduce a regularity condition on $\mathring{\alpha}_i$ as below.
\begin{itemize}[leftmargin=2mm] \itshape
\item[] (M) Given representative training $pq$-design, $\mathring{\alpha}_i$ is a function given as
\[
\mathring{\alpha}_i = m_i\Big( y_i, x_i, E_{pq}\big( \mathring{\mu}(x_i, s) \mid i \notin s_1 \big) \Big)
\]
\end{itemize}
Note that the function $m_i(\cdot)$ is allowed to vary across the units. The condition (M) effectively requires $\mathring{\alpha}_i$ to be a constant given the arguments of $m_i(\cdot)$. 

By \eqref{unbiasMC}, we can exchange $E_{pq}\big( \mathring{\mu}(x_i, s) \mid i \notin s \big)$ and $E_{pq}\big( \mathring{\mu}(x_i, s) \mid i \in s_2 \big)$, such that the mapping $m_i(\cdot)$ of the condition (M) yields 
\begin{equation} \label{M-pq}
\mathring{\alpha}_i = E_{pq}\Big( \mathbb{I}\big( y_i = \mathring{y}_i(s)\big) \mid i\in s_2 \Big) 
= E_{pq}\Big( \mathbb{I}\big( y_i = \mathring{y}_i(s)\big) \mid i\notin s \Big) .
\end{equation}

Notice that the condition (M) is not trivial since, as the Lemma below shows: although each $\mathbb{I}\big( y_i = \mathring{y}_i(s)\big)$ is a function of $\mathring{\mu}(x_i, s)$, apart from $(y_i, x_i)$ obviously, it is not necessary that $\mathring{\alpha}_i$ then only depends on $E_{pq}\big( \mathring{\mu}(x_i, s) \mid i\notin s_1 \big)$.

\begin{lemma} \label{lem:2}
The random classifier $\mathring{y}_i(s) = \mathbb{I}\big(u_i \leq \mathring{\mu}(x_i, s)\big)$, where $u_i \overset{\text{IID}}{\sim} \text{Unif}(0,1)$, satisfies the condition (M). The deterministic classifier $\mathring{y}_i(s) = \mathbb{I}\big( \mathring{\mu}(x_i, s) \geq \psi \big)$ does not satisfy the condition (M) given any threshold value $\psi \in (0,1)$.
\end{lemma}

\begin{proposition} \label{OOB-cls}
Given a representative training $pq$-design and a classifier satisfying the condition (M), a $pq$-unbiased predictor of the out-of-sample classification accuracy $\mathring{\Phi}(R)$ in \eqref{cls} is given as
\begin{equation} \label{cls-est}
\mathring{\Phi}_w(s) = \sum_{i \in s} \big( \tfrac{1}{\pi_i} -1 \big) \mathbb{I}\big( y_i = \mathring{y}_i(s) \big)~.  
\end{equation}
\end{proposition}

Notice that we do not estimate each $\mathring{\alpha}_i$ specifically, but rather we work with their aggregates. The result above includes the exact SRB-probability $\bar{\mu}(x_i, s)$ as a special case when it is feasible instead of $\mathring{\mu}(x_i, s)$.

Notice also that the use of in-sample out-of-bag $\mathbb{I}\big( y_i = \mathring{y}_i(s) \big)$ in \eqref{cls-est} means that debiasing $\mu(x_i, s_1)$ by tuning needs to be restricted to $s_1$, such as kNN based on donor set $s_{11}$ and tuned by $s_1\setminus s_{11}$, where $s_{11} \subset s_1$. This is because tuning must not invalidate the `out-of-bag-ness' of $\mathring{y}_i$, i.e. tuning must be based on the units outside each $s_2$ --- hence, inside each $s_1$.

\section{Illustrative application} \label{sec:application}

We apply prediction tuning and classification accuracy estimation to a real dataset to demonstrate the theory and methods developed in this paper. 

\subsection{Data and setup} \label{sec:setup}

Let $N=10^4$ satellite images be randomly selected from the CropHarvest (HDF5) dataset, each assigned $y=1$ in case of a coffee field or $y=0$ otherwise, associated with an $x$-vector of 216 features. We treat these images as the population $U$ with associated constants $\{ (y_i, x_i) : i=1, ..., N\}$ over repeated sampling. 

Let kNN with $k=5$ be the ML algorithm, based on any of the three feature sets below: 
\begin{itemize}[leftmargin=4mm,itemsep=0pt] \setstretch{-1.0} 
\item[] I. (x\_126, x\_144, x\_180, x\_190, x\_198) by LASSO logistic regression;
\item[] II. (x\_008, x\_025, x\_048, x\_129, x\_183) by random selection;
\item[] III. (x\_043, x\_162, x\_136, x\_063, x\_024) by forward selection for 5NN.
\end{itemize}
Consider the following 5NN-based predictors for out-of-sample $y$-classification. 
\begin{itemize}[leftmargin=6mm, itemsep=0pt] 
\item Sample-5NN $\eta(x,s)$, given sample $s$ of size $n$ selected by SRS.
\item SRB-5NN $E_q\big( \eta(x,s_1)\mid s\big)$, with training set $s_1$ of size $n_1$ by SRS from $s$.
\item OOB-tuned SRB-5NN $E_q\big( \eta_{\tau}(x, s_1) \mid s\big)$, where $\eta_{\tau}(x, s_1)$ is the $s_{11}$-5NN that is tuned by out-of-$s_{11}$ errors in $s_1\setminus s_{11}$, given $s_{11}$ of size $n_{11}$ by SRS from $s_1$, i.e.
\[
\eta_{\tau}(x, s_1) = \eta(x, s_{11}) - \tfrac{1}{n_1 - n_{11}} \sum_{i\in s_1\setminus s_{11}} \big( \eta(x_i, s_{11}) - y_i \big) 
\]
\end{itemize}

We use simulations to evaluate the performance of design-based debiasing and classification accuracy prediction. Let $n=200$ be the size of sample $s$ by SRS from $U$, let $n_1 = 120$ be the size of $s_1$ for SRB given by SRS from $s$, and let $n_{11} = 80$ be the size of $s_{11}$ for the OOB-tuned SRB given by SRS from $s_1$. Let $T=100$ for Monte Carlo SRB, such that, for any $k\in U$, the SRB-5NN predictor is given by \eqref{OOB-SRB} generically, and the OOB-tuned SRB-5NN predictor is given by 
\[
\mathring{\eta}_{\tau}(x_k, s) = \frac{1}{T_k} \sum_{t=1}^T \mathbb{I}(k\notin s_1^{(t)}) 
\Big\{ \eta(x_k, s_{11}^{(t)}) - \tfrac{1}{n_1 - n_{11}} \sum_{i\in s_1^{(t)}\setminus s_{11}^{(t)}} \big( \eta(x_i, s_{11}^{(t)}) - y_i \big)  \Big\} .
\]
Finally, for the classification accuracy \eqref{cls} by randomised classifier associated with a given OOB SRB-5NN probability, denoted by $\mathring{\eta}(x_i, s)$, we can replace the random variable $\mathbb{I}\big( y_i = \mathring{y}_i(s) \big)$ in \eqref{cls-est} by 
\[
p_i = \mathbb{I}(y_i =1) \mathring{\eta}(x_i, s) + \mathbb{I}(y_i =0) \big( 1- \mathring{\eta}(x_i, s) \big) ~.
\]
This is more efficient than using the random $\mathring{y}_i(s) = \mathbb{I}\big( u_i \leq \mathring{\eta}(x_i, s) \big)$, which is subject to an additional variance due to $u_i \overset{\text{IID}}{\sim} \text{Unif}(0,1)$.

\subsection{Results}

Denote by $B$ the number of simulations, each corresponding to a different sample $s$ from $U$. Let $\hat{P}^{(b)}$ be a realised prediction estimate of the population mean $P$, as defined in Section \ref{sec:kNN}, where $b=1, ..., B$. Let
\begin{gather*}
\text{EBias} = \frac{1}{B} \sum_{b=1}^B (\hat{P}^{(b)} - P) \quad\text{with}\quad
\text{V(EBias)} = \frac{1}{B(B-1)} \sum_{b=1}^B (\hat{P}^{(b)} - P)^2 \\
\quad\text{and}\quad \text{EVar} = \frac{1}{B-1} \sum_{b=1}^B \{ \hat{P}^{(b)} - \tfrac{1}{B} \sum_{b=1}^B \hat{P}^{(b)} \}^2~.
\end{gather*}

\begin{table}[ht]
\centering
\caption{Debiasing results, feature set (I)--(III), $B=100$ simulations} 
\begin{tabular}{lrrr} \toprule
Predictor, feature set (I) & EBias & $\sqrt{\text{V(EBias)}}$ & EVar \\ \midrule
Sample-5NN & -0.01712	& 0.00192 & 0.000072 \\
SRB-5NN & -0.01776 & 0.00200 & 0.000081 \\
OOB-tuned SRB-5NN & -0.00115 & 0.00130 & 0.000166 \\ \bottomrule
Predictor, feature set (II) & EBias & $\sqrt{\text{V(EBias)}}$ & EVar \\ \midrule
Sample-5NN & 0.00063 & 0.00079 & 0.000061 \\
SRB-5NN & -0.00236 & 0.00091 & 0.000076 \\
OOB-tuned SRB-5NN & 0.00008 & 0.00104 & 0.000107 \\ \bottomrule
Predictor, feature set (III) & EBias & $\sqrt{\text{V(EBias)}}$ & EVar \\ \midrule
Sample-5NN & 0.00234 & 0.00043 & 0.000013 \\
SRB-5NN & 0.00422 & 0.00061 & 0.000019 \\
OOB-tuned SRB-5NN & -0.00035 & 0.00042 & 0.000017 \\ \bottomrule
\end{tabular} \label{tab:bias}
\end{table}

Table \ref{tab:bias} shows the results for debiasing, based on $B=100$ simulations. Firstly, given any feature set, the OOB-tuned SRB-5NN is essentially unbiased for the population mean: its EBias is clearly closer to 0 than the untuned SRB-5NN and not significantly different from 0 in light of V(EBias). This confirms again the theory of design-based debiasing developed in Section \ref{sec:prediction}. 

Secondly, the bias of kNN prediction depends on the selected features and their sample distribution. In this case, the randomly selected feature set (II) causes very little bias to the standard sample-5NN, which is possible if it leads to ‘random’ donor imputation under SRS. However, such `accidental' unbiasedness can come at a price of increased variance, which is evident if one compares the EVar of sample-5NN given either feature set (II) or (III). We notice that the results of feature set (II) are included here only for instructive purposes, since one is unlikely to use randomly selected features for kNN in practice. 

Thirdly, feature selection is an integral part of algorithmic ML. One can see this most clearly in the poor results given feature set (I). Although it may be a reasonable choice for logistic regression, using these features for kNN can increase both the prediction bias and variance. In this case, the EBias of sample-5NN (or SRB-5NN) is larger than $\sqrt{\text{EVar}}$ and the EVar is even larger than that given the randomly selected feature set (II).

Finally, given feature set (III) for kNN, the EBias of directly applying 5NN is small but still appreciable compared to $\sqrt{\text{EVar}}$. Since the empirical MSE of the OOB-tuned SRB-5NN is smaller than either the untuned 5NN predictors, debiasing is worthwhile for out-of-sample prediction here.

\begin{table}[ht]
\centering
\caption{Mean over simulations ($B=100$) of classification accuracy  $\mathring{\phi}(R)$, its predictor $\hat{\phi}(R)$, and associated squared error $\{ \hat{\phi}(R) - \mathring{\phi}(R) \}^2$} 
\begin{tabular}{lrrr} \toprule
& \multicolumn{3}{c}{Average over $B$ simulations} \\ \cline{2-4}
Predictor, feature set (I) & $\mathring{\phi}(R)$ & $\hat{\phi}(R)$ & $\{ \hat{\phi}(R) - \mathring{\phi}(R) \}^2$\\ \midrule
SRB-5NN & 0.944 & 0.945 & 0.000143 \\
OOB-tuned SRB-5NN & 0.922 & 0.923 & 0.000141 \\ \bottomrule
Predictor, feature set (II) & $\mathring{\phi}(R)$ & $\hat{\phi}(R)$ & $\{ \hat{\phi}(R) - \mathring{\phi}(R) \}^2$\\ \midrule
SRB-5NN & 0.962 & 0.963 & 0.000102 \\
OOB-tuned SRB-5NN & 0.939 & 0.941 & 0.000106 \\ \bottomrule
Predictor, feature set (III) & $\mathring{\phi}(R)$ & $\hat{\phi}(R)$ & $\{ \hat{\phi}(R) - \mathring{\phi}(R) \}^2$\\ \midrule
SRB-5NN & 0.989 & 0.989 & 0.000018 \\
OOB-tuned SRB-5NN & 0.987 & 0.987 & 0.000020 \\ \bottomrule
\end{tabular} \label{tab:accuracy}
\end{table}

When it comes to the classification accuracy $\mathring{\phi}(R)$ associated with a given SRB-5NN predictor, as defined by \eqref{cls}, let $\hat{\phi}(R) = n_R^{-1} \mathring{\Phi}_w(s)$ according to \eqref{cls-est}. Table \ref{tab:accuracy} gives the results based on $B=100$ simulations, where the EBias of classification accuracy prediction is the difference between columns $\mathring{\phi}(R)$ and $\hat{\phi}(R)$, and the last column shows its empirical MSE over the simulations. It is evident that Proposition \ref{OOB-cls} achieves unbiased prediction of the out-of-sample classification accuracy in every situation.

Notice that, given either feature set (I) or (II), the classification accuracy of the unbiased OOB-tuned SRB-5NN is clearly reduced compare to the biased SRB-5NN. The reason is that tuning can cause the predicted probabilities to be out of the bounds $[0,1]$, which then leads to truncation of probabilities and loss of classification accuracy, as illustrated by the example below. 

\begin{example} Suppose two out-of-sample units with $(y_1, y_2) = (1,0)$. Suppose a given predictor yields $(\mu_1, \mu_2) = (0.95, 0.07)$, such that the probability of correct randomised classification is $\tfrac{1}{2} (0.95 + 0.93) = 0.94$. Given $\mu_1 + \mu_2 = 1.02 \neq y_1 + y_2$, tuning is possible, which yields, say, $(\mathring{\mu}_1, \mathring{\mu}_2) = (\mu_1, \mu_2) + \epsilon$. As long as $(\mathring{\mu}_1, \mathring{\mu}_2)$ are within the bounds $[0,1]$, the probability of correct randomised classification is unaffected, since $\tfrac{1}{2} (\mu_1+\epsilon + 1- \mu_2 -\epsilon) = \tfrac{1}{2} (\mu_1+ 1- \mu_2)$. However, it is possible to obtain out-of-bound probabilities, say, $(\mathring{\mu}_1, \mathring{\mu}_2) = (1.04, 0.16)$, in which case the probability of correct randomised classification would be reduced to $\tfrac{1}{2} (1 + 0.84) = 0.92$. 
\end{example}

Meanwhile, given the feature set (III) selected for kNN, the drop of classification accuracy by debiasing is much less pronounced, i.e. 0.987 compared to 0.989, and the classification accuracy is higher compared to using the feature set (I) or (II). In other words, not only does sensible feature selection improve the classification accuracy, it can also reduce the risk of out-of-bounds tuning for the purpose of unbiased population prediction estimation.

\section{Summary of findings and future research} \label{sec:final}

ML must use a training dataset to form the prediction model or algorithm. It would be difficult to relate its in-sample out-of-bag prediction error of a given unit to its out-of-sample prediction error of the same unit, if the \emph{expected} prediction would differ depending on whether the unit is in or out of the sample. The concept of representative training articulates this intuition in terms of how the training set can be selected from the relevant finite population, whether it consists of persons, businesses, physical or spatial objects. 

Notice that the same intuition has always existed in IID-model inference, such as cross-validation, or out-of-bag accuracy for random forests. Since one can obtain IID samples by with-replacement sampling from a given population, the concept of representative training can be viewed as a generalisation of these IID-model inference ideas to finite population inference, where the units may be selected with unequal probabilities and without replacement. 

Moreover, representative training provides generally a means for tuning a trained algorithm, based on its in-sample out-of-bag prediction errors, so as to achieve design-unbiased out-of-sample prediction, i.e. over repeated sampling-and-training under the $pq$-design. A topic for future study may be to investigate representative training for interval estimation so as to achieve exact design-based coverage as it was defined by Neyman (1934). 

In addition, we have shown that out-of-sample unit-level classification of categorical outcomes can yield unbiased prediction estimation of population totals, and the associated classification accuracy can be assessed unbiasedly given representative training by the $pq$-design.
 
However, one should keep in mind the distinctive needs of population-level estimation and unit-level classification. For instance, in many applications of ML, such as official statistics, population estimation is always required, but unit classification is only sometimes. If unbiasedness of population estimation is necessary, then one may have to accept a reduced classification accuracy. This seems therefore a natural topic for future research, i.e. how to implement unit-level classification subject to unbiased population prediction estimation, such that the classification accuracy can be maximised.  

Finally, looking back on the groundbreaking idea of Neyman (1934) to finite population sampling, which consists of a ``representative'' method of sampling and an associated ``consistent'' method of estimation, one may view representative training $pq$-designs as an extension of the concept of ``representative'' sampling, providing an intuitive basis for broadening the scope of associated ``consistent'' ML estimation by predicting or classifying the out-of-sample units. We hope this general outlook on design-unbiased finite population inference may be useful to many practitioners of ML.

\appendix 
\section{Proofs} \label{app:proof}

Lemma \ref{RT:lemma}.
\begin{proof}
For the proof here, write $\pi_{2i}(s_1)$ for $\pi_{2i}$ to emphasise it as a function of $s_1$.
Given a well-defined $pq$-design, we have $E_{pq}(\cdot) = E_{s_1s_2}(\cdot)$, such that
\begin{gather}
\eqref{RT} \quad\Leftrightarrow\quad 
\sum_{s_1\not\ni i} \mu(x_i, s_1) \frac{f(s_1) \Pr(i\in s_2 \mid s_1)}{\Pr(i\in s_2)}  =
\sum_{s_1\not\ni i} \mu(x_i, s_1) \frac{f(s_1) \Pr(i\notin s_2 \mid s_1)}{\Pr(i\notin s)} \notag \\
\quad\Leftrightarrow\quad 
\sum_{s_1\not\ni i} \mu(x_i, s_1) f(s_1) \frac{\pi_{2i}(s_1)}{\pi_i - \Pr(i\in s_1)}  =
\sum_{s_1\not\ni i} \mu(x_i, s_1) f(s_1) \frac{1- \pi_{2i}(s_1)}{1-\pi_i} \notag \\
\quad\Leftrightarrow\quad 
\sum_{s_1\not\ni i} \mu(x_i, s_1) f(s_1) \pi_{2i}(s_1) \frac{1 - \Pr(i\in s_1)}{\pi_i - \Pr(i\in s_1)} =
\sum_{s_1\not\ni i} \mu(x_i, s_1) f(s_1) . \label{eq:tmp}
\end{gather}
First, the equality \eqref{eq:tmp} holds for \emph{all possible} $\mu(x_i, s_1)$ if and only if
\[
\pi_{2i}(s_1) \frac{1 - \Pr(i\in s_1)}{\pi_i - \Pr(i\in s_1)} = 1 \quad\Leftrightarrow\quad \eqref{RT:all}
\]
Next, multiplying both the sides of \eqref{eq:tmp} by $\tfrac{\pi_i - \Pr(i\in s_1)}{\{1 - \Pr(i\in s_1)\}^2}$, the left-hand side yields
\[
\sum_{s_1\not\ni i} \mu(x_i, s_1) \pi_{2 i}(s_1) \frac{f(s_1)}{1 - \Pr(i\in s_1)} = E_{s_1}\big( \mu(x_i, s_1) \pi_{2 i} (s_1) \mid i \notin s_1 \big) 
\]
while the right-hand side can be written as
\begin{gather*}
\Big( \sum_{s_1\not\ni i} \mu(x_i, s_1) \tfrac{f(s_1)}{1 - \Pr(i\in s_1)} \Big) \Big( \tfrac{\Pr(i\in s_2)}{1 - \Pr(i\in s_1)} \Big) 
= E_{s_1}\big( \mu(x_i, s_1) \mid i \notin s_1 \big) \Big( \sum_{s_1\not\ni i} \pi_{i2}(s_1) \tfrac{f(s_1)}{1 - \Pr(i\in s_1)} \Big) \\
= E_{s_1}\big( \mu(x_i, s_1) \mid i \notin s_1 \big) E_{s_1}\big( \pi_{2i}(s_1) \mid i \notin s_1 \big) 
\end{gather*}
i.e. representative training of any given $\mu(x, s_1)$ if and only if \eqref{RT:given}, including when $\pi_{2i}(s_1)$ varies with $s_1$ and is not given by \eqref{RT:all}.
\end{proof}

\noindent
Corollary \ref{SRS-SRS}.
\begin{proof} Given sample sizes $n$ of $s$ and $n_1$ of $s_1$, we have $\mathcal{C}(N-1,n_1) =\frac{(N-1)!}{n_1! (N-n_1-1)!}$ distinct $s_1$ on either side of \eqref{RT}, and $\Pr(i\in s_2 \mid s_1) = \tfrac{n-n_1}{N-n_1} \mathbb{I}(i\notin s_1)$. 
\end{proof}

\noindent
Corollary \ref{Pois-Ber}.
\begin{proof} Let $\pi_i$ be the Poisson sampling probability of $i\in U$, and $\lambda$ that of Bernoulli sampling from $s$.  We have $f(s_1) = \prod_{j\in s_1} \lambda \pi_j \prod_{l\notin s_1} (1-\lambda \pi_l)$ and $\Pr(i\in s_1) = \lambda \pi_i$. Moreover, we have $\Pr(i\in s_2 \mid s_1) = \tfrac{\pi_i - \lambda\pi_i}{1-\lambda \pi_i} \mathbb{I}(i\notin s_1)$ for each $i\in U$. \end{proof}

\noindent
Proposition \ref{OOB}.
\begin{proof} Regarding the right-hand sides in \eqref{OOB-q} and \eqref{unbias-pq}, we have
\[
E_p\Big( \sum_{i\in s} \big( \tfrac{1}{\pi_i} -1 \big) y_i \Big) = \sum_{i\in U} y_i - E_p\big( \sum_{i\in s} y_i \big)
= E_p\big( \sum_{i\in R} y_i \big) .
\]
Meanwhile, regarding the left-hand sides in \eqref{OOB-q} and \eqref{unbias-pq}, we have
\begin{align*}
E_{pq}\Big( \sum_{i\in s} \big(\tfrac{1}{\pi_i} -1\big) \mu(x_i, s_1) \mid i \in s_2\Big) & = 
E_{pq}\Big( \sum_{i\in U} \tfrac{\mathbb{I}(i\in s)}{\pi_i} (1- \pi_i) \mu(x_i, s_1) \mid i \in s_2\Big) \\
& = \sum_{i\in U} E_{pq}\Big( (1- \pi_i) \mu(x_i, s_1) \mid i \in s_2, i\in s \Big) \\
& = \sum_{i\in U} (1- \pi_i) E_{pq}\big( \mu(x_i, s_1) \mid i \in s_2 \big) \\
& \overset{\eqref{RT}}{=} \sum_{i\in U} (1- \pi_i) E_{pq}\big( \mu(x_i, s_1) \mid i \notin s \big) \\
& = E_{pq}\big( \sum_{i\in R} \mu(x_i, s_1) \big) 
\end{align*}
given representative training $pq$-design. Thus, the condition \eqref{OOB-q} ensures \eqref{unbias-pq}. 
\end{proof}

\noindent
Proposition \ref{prop:OOB-tuned}.
\begin{proof} Since the total out-of-sample prediction bias of $\bar{\mu}(x_i, s)$ is given by the expectation of the difference between the two sides of  \eqref{OOB-q}, one can apportion $1/n_R$ of this difference to each $i\notin s$ as a $p$-unbiased correction. The expression \eqref{OOB-tuned}  follows, since
\[
E_q\Big( \sum_{i\in s_2} \tfrac{1}{q_{2i}} \big( \tfrac{1}{\pi_i} -1\big) \mu(x_i, s_1) \Big) =
\sum_{i\in s} \big( \tfrac{1}{\pi_i} -1\big) E_q\big( \mu(x_i, s_1) \mid i\in s_2\big) . \qedhere
\]
\end{proof}

\noindent
Lemma \ref{lem:2}.
\begin{proof} For the random classifier, we have
\begin{gather*}
\mathbb{I}\big( y_i = 1 = \mathring{y}_i(s)\big) = 
\begin{cases} 1 & \text{if}~ u_i\leq \mathring{\mu}(x_i, s) \\ 0 & \text{if}~ u_i> \mathring{\mu}(x_i, s) \end{cases} 
~\Rightarrow~ \mathring{\alpha}_i = E_{pq}\big( \mathring{\mu}(x_i, s) \mid i\notin s_1\big) \text{ if } y_i=1 \\
\mathbb{I}\big( y_i = 0 = \mathring{y}_i(s)\big) = 
\begin{cases} 1 & \text{if}~ u_i> \mathring{\mu}(x_i, s) \\ 0 & \text{if}~ u_i \leq \mathring{\mu}(x_i, s) \end{cases}
~\Rightarrow~ \mathring{\alpha}_i = 1- E_{pq}\big( \mathring{\mu}(x_i, s) \mid i\notin s_1\big) \text{ if } y_i=0 
\end{gather*}
such that $\mathring{\alpha}_i$ is fully determined given $y_i$, $x_i$ and $E_{pq}\big( \mathring{\mu}(x_i, s) \mid i\notin s_1 \big)$.

Whereas, for the deterministic classifier, we have 
\begin{gather*}
\mathbb{I}\big( y_i = 1 = \mathring{y}_i(s)\big) = 
\begin{cases} 1 & \text{if}~ \mathring{\mu}(x_i, s) \geq \psi \\ 0 & \text{otherwise} \end{cases} 
~\Rightarrow~ \mathring{\alpha}_i = \Pr\big( \mathring{\mu}(x_i, s) \geq \psi \mid i\notin s_1\big) \text{ if } y_i=1 \\
\mathbb{I}\big( y_i = 0 = \mathring{y}_i(s)\big) = 
\begin{cases} 1 & \text{if}~ \mathring{\mu}(x_i, s) < \psi \\ 0 & \text{otherwise} \end{cases}
~\Rightarrow~ \mathring{\alpha}_i = \Pr\big( \mathring{\mu}(x_i, s) < \psi \mid i\notin s_1\big) \text{ if } y_i=0
\end{gather*}
such that $\mathring{\alpha}_i$ depends on the variance of $\mathring{\mu}(x_i, s)$ as well as $E_{pq}\big( \mathring{\mu}(x_i, s) \mid i\notin s_1 \big)$.
\end{proof}

\noindent
Proposition \ref{OOB-cls}.
\begin{proof} We have 
\begin{align*}
E_{pq}\big( \mathring{\Phi}(R) \big) 
& = \sum_{i\in U} E_{pq}\Big( \mathbb{I}(i\notin s) \mathbb{I}\big( y_i = \mathring{y}_i(s)\big)  \Big) \\
& = \sum_{i\in U} (1- \pi_i) E_{pq}\Big(\mathbb{I}\big( y_i = \mathring{y}_i(s)\big) \mid i\notin s \Big) 
\overset{\eqref{M-pq}}{=} \sum_{i\in U} (1- \pi_i) \mathring{\alpha}_i \\
E_{pq}\big( \mathring{\Phi}_w(s) \big) & =
\sum_{i\in U} E_{pq}\big( \mathbb{I}(i\in s) \big) \big( \tfrac{1}{\pi_i} -1 \big) 
E_{pq}\Big(\mathbb{I}\big( y_i = \mathring{y}_i(s)\big) \mid i\in s_2, i\in s \Big) \\
& = \sum_{i\in U} (1 -\pi_i) E_{pq}\Big(\mathbb{I}\big( y_i = \mathring{y}_i(s)\big) \mid i\in s_2 \Big) 
\overset{\eqref{M-pq}}{=} \sum_{i\in U} (1- \pi_i) \mathring{\alpha}_i . \qedhere
\end{align*}
\end{proof}

\section{Details of Example \ref{ex:toy}} \label{app:toy}

Let $Y_R = (N-n) \bar{Y}_R$. Provided $n\geq 3$, we have $\eta(x_i, s) =0$ for any $i\neq 1, 2$, and
\[
\hat{Y}_R(s) - Y_R \overset{n\geq 3}{=} \begin{cases} 
-1 & \text{if $1\notin s$, \hspace{10mm} $\sum_{i\notin s} \eta(x_i, s) = 0$, prob. $\frac{N-n}{N}$} \\
0 & \text{if $1\in s, 2\in s$, $\sum_{i\notin s} \eta(x_i, s) = 0$, prob. $\frac{n(n-1)}{N(N-1)}$} \\
+1 & \text{if $1\in s, 2\notin s$, $\sum_{i\notin s} \eta(x_i, s) = 1$, prob. $\frac{N-n}{N} \frac{n}{N-1}$} . \end{cases}
\]
Thus, using $\eta(x, s)$ is $pq$-biased for estimating $Y_R$ except when $n = N-1$.

Meanwhile, let $\hat{Y}_R(s_1)$ be the prediction estimator of $Y_R$ using $\eta(x_i, s_1)$ for any $i\notin s$. Provided $n_1 \geq 3$, we have $\eta(x_i, s_1) =0$ for any $i\neq 1, 2$, and
\[
\hat{Y}_R(s_1) - Y_R \overset{n_1\geq 3}{=} \begin{cases} 
-1 & \text{if $1\notin s$, \hspace{11mm} $\sum_{i\notin s} \eta(x_i, s_1) = 0$, prob. $\frac{N-n}{N}$} \\
0 & \text{if $1\in s, 2\in s$,~ $\sum_{i\notin s} \eta(x_i, s_1) = 0$, prob. $\frac{n(n-1)}{N(N-1)}$} \\
+1 & \text{if $1\in s_1, 2\notin s$, $\sum_{i\notin s} \eta(x_i, s_1) = 1$, prob. $\frac{N-n}{N} \frac{n_1}{N-1}$} \\
0 & \text{if $1\in s_2, 2\notin s$, $\sum_{i\notin s} \eta(x_i, s_1) = 0$, prob. $\frac{N-n}{N} \frac{n_2}{N-1}$} . \end{cases}
\]
Thus, using the $\bar{\eta}(x, s)$ predictor also leads to $pq$-biased estimation generally.

For residual-tuning of $\eta(x, s_1)$ by $c_{\eta}(s_1) = \frac{1}{n} \sum_{i\in s} \eta(x_i, s_1) -\bar{y}_s$. We have
\[
c_{\eta}(s_1) \overset{n_1\geq 3}{=} \begin{cases} 
0 & \text{if $1\notin s$, \hspace{11mm} $\sum_{i\in s} \eta(x_i, s_1) =0$, $\sum_{i\in s} y_i = 0$, prob. $\frac{N-n}{N}$} \\
\frac{1}{n} & \text{if $1\in s_1, 2\in s$, $\sum_{i\in s} \eta(x_i, s_1) =2$, $\sum_{i\in s} y_i = 1$, prob. $\frac{n_1(n-1)}{N(N-1)}$} \\
-\frac{1}{n} & \text{if $1\in s_2, 2\in s$, $\sum_{i\in s} \eta(x_i, s_1) =0$, $\sum_{i\in s} y_i = 1$, prob. $\frac{n_2(n-1)}{N(N-1)}$} \\
0 & \text{if $1\in s_1, 2\notin s$, $\sum_{i\in s}\eta(x_i, s_1) =1$, $\sum_{i\in s} y_i =1$, prob. $\frac{N-n}{N} \frac{n_1}{N-1}$} \\
-\frac{1}{n} & \text{if $1\in s_2, 2\notin s$, $\sum_{i\in s} \eta(x_i, s_1) =0$, $\sum_{i\in s} y_i =1$, prob. $\frac{N-n}{N} \frac{n_2}{N-1}$} 
\end{cases}
\]
such that the SRB residual-tuning of $\bar{\eta}(x, s)$ is given by
\[
\bar{c}_{\eta}(s) \overset{n_1\geq 3}{=} \begin{cases} 
0 & \text{if $1\notin s$, \hspace{9.5mm} prob. $\frac{N-n}{N}$} \\
\frac{n_1 -n_2}{n^2} & \text{if $1\in s, 2\in s$, prob. $\frac{n(n-1)}{N(N-1)}$} \\
-\frac{n_2}{n^2} & \text{if $1\in s, 2\notin s$, prob. $\frac{N-n}{N} \frac{n}{N-1}$} .
\end{cases}
\]
Straightforward algebra shows that this does \emph{not} remove the bias, 
\[
\frac{1}{N-n} E_{pq}\big( \hat{Y}_R(s_1) - Y_R \big) - E_p\big( \bar{c}_{\eta}(s)\big) \neq 0 .
\]

For out-of-bag tuning by $\tau_{\eta}(s_1) = \frac{1}{n_2} \sum_{i\in s_2} \big( \eta(x_i, s_1) - y_i \big)$. We have
\[
\tau_{\eta}(s_1) \overset{n_1\geq 3}{=} \begin{cases} 
0 & \text{if $1\notin s$, \hspace{12.5mm} $\sum_{i\in s_2} \eta(x_i, s_1) =0$, $\sum_{i\in s_2} y_i = 0$, prob. $\frac{N-n}{N}$} \\
-\frac{1}{n_2} & \text{if $1\in s_2, 2\in s$,~ $\sum_{i\in s_2} \eta(x_i, s_1) =0$, $\sum_{i\in s_2} y_i = 1$, prob. $\propto \frac{n_2}{n}$} \\
0 & \text{if $1\in s_1, 2\in s_1$, $\sum_{i\in s_2} \eta(x_i, s_1) =0$, $\sum_{i\in s_2} y_i = 0$, prob. $\propto \frac{n_1(n_1-1)}{n(n-1)}$} \\
\frac{1}{n_2} & \text{if $1\in s_1, 2\in s_2$, $\sum_{i\in s_2} \eta(x_i, s_1) =1$, $\sum_{i\in s_2} y_i = 0$, prob. $\propto \frac{n_2n_1}{n(n-1)}$} \\
0 & \text{if $1\in s_1, 2\notin s$,~ $\sum_{i\in s_2}\eta(x_i, s_1) =0$, $\sum_{i\in s_2} y_i =0$, prob. $\frac{N-n}{N} \frac{n_1}{N-1}$} \\
-\frac{1}{n_2} & \text{if $1\in s_2, 2\notin s$,~ $\sum_{i\in s_2} \eta(x_i, s_1) =0$, $\sum_{i\in s_2} y_i =1$, prob. $\frac{N-n}{N} \frac{n_2}{N-1}$} .
\end{cases}
\]
such that the SRB out-of-bag tuning of $\bar{\eta}(x, s)$ is given by
\[
\bar{\tau}_{\eta}(s) \overset{n_1\geq 3}{=} \begin{cases} 
0 & \text{if $1\notin s$, \hspace{9.5mm} prob. $\frac{N-n}{N}$} \\
-\frac{n_2-1}{n(n-1)} & \text{if $1\in s, 2\in s$, prob. $\frac{n(n-1)}{N(N-1)}$} \\
-\frac{1}{n} & \text{if $1\in s, 2\notin s$, prob. $\frac{N-n}{N} \frac{n}{N-1}$} .
\end{cases}
\]
Straightforward algebra to shows that this does remove the bias, 
\[
\frac{1}{N-n} E_{pq}\big( \hat{Y}_R(s_1) - Y_R \big) - E_p\big( \bar{\tau}_{\eta}(s)\big) = 0 .
\]

\section{Design-based bias and MSE} \label{app:inference}

The derivation follows the approach of Zhang et al. (2025). 

\paragraph{Bias of SRB-kNN} The prediction error due to $\eta(x, s_1)$ can be written as
\[
B(s_1) = \frac{1}{N} \sum_{i\notin s} e_i(s_1) \quad\text{and}\quad e_i(s_1) = \eta(x_i, s_1) - y_i .
\]
Given $\pi_{2i} = \frac{n_2}{N-n_1}$ under SRS-SRS, an unbiased predictor based on $s_2$ is 
\[
\hat{B}(s_1) = \frac{N-n}{N n_2} \sum_{i\in s_2} e_i(s_1) \quad\text{where}\quad 
E_s\big( \hat{B}(s_1) - B(s_1) \mid s_1\big) = 0 .
\]
For the corresponding error of the SRB-kNN $\bar{\eta}(x, s)$, denoted by $\bar{B}(s)$, we have  
\begin{equation} \label{biasSRBkNN}
\hat{\bar{B}}(s) = \frac{N-n}{N n_2} E_q \Big( \sum_{i\in s_2} e_i(s_1)\Big) 
=\frac{N-n}{N n} \sum_{i\in s} \big( \dot{\eta}(x_i, s) - y_i \big) .
\end{equation}

\paragraph{Bias of residual-tuned SRB-kNN} The prediction error due to $\eta(x, s_1)$ is
\begin{align*}
B_c(s_1) & = \frac{1}{N} \sum_{i\notin s} e_i(s_1) + \frac{N-n}{N} \frac{1}{n} \sum_{i\in s} \big( y_i - \eta(x_i, s_1) \big) \\
& = \frac{1}{N} \sum_{i\notin s} e_i(s_1) - \Big( \frac{1}{n} - \frac{1}{N} \Big) \Big( \sum_{i\in s_1} e_i(s_1) + \sum_{i\in s_2} e_i(s_1) \Big) \\
& = \frac{N-n}{N} \bar{e}_{U\setminus s}(s_1) 
- \Big( \frac{1}{n} - \frac{1}{N} \Big) \Big( n_1 \bar{e}_{s_1}(s_1) + n_2 \bar{e}_{s_2}(s_1) \Big)
\end{align*}
where $\bar{e}_D(s_1)$ denotes the average of $e_i(s_1)$ over $i\in D$ given any set $D$. Conditional on $s_1$, an unbiased predictor of $B_c(s_1)$ follows as
\begin{align*}
\hat{B}_c(s_1) & = \Big( 1 - \frac{n}{N} \Big) \bar{e}_{s_2}(s_1)
- \Big( \frac{1}{n} - \frac{1}{N} \Big) \Big( n_1 \bar{e}_{s_1}(s_1) + n_2 \bar{e}_{s_2}(s_1) \Big) \\
& = n_1 \Big( \frac{1}{n} - \frac{1}{N} \Big) \Big( \bar{e}_{s_2}(s_1) - \bar{e}_{s_1}(s_1) \Big)
\end{align*}
such that an unbiased estimator of the bias of residual-tuned SRB-kNN is
\begin{equation} \label{biasResidTuned}
\hat{\bar{B}}_c(s) = E_q\big( \hat{B}_c(s_1) \mid s \big) .
\end{equation}

\paragraph{MSE of SRB-kNN} The error $B(s_1)$ has been given earlier. Conditional on $s_1$, an unbiased estimator of the conditional MSE of $\hat{P}(s_1)$ can be given as
\[
\mbox{mse}(s_1) = \hat{B}(s_1)^2 - \hat{V}_s\{ \hat{B}(s_1) \mid s_1\} + \hat{V}_s\{ B(s_1) \mid s_1\}
\]
(Theorem 1, Zhang et al., 2025), where
\begin{gather*}
\hat{V}_s\{ \hat{B}(s_1) \mid s_1\} = \Big( 1 - \frac{n}{N} \Big)^2 v_2 \\
v_2 = \hat{V}_s\big( \bar{e}_{s_2}(s_1)\mid s_1\big) 
= \Big( \frac{1}{n_2} - \frac{1}{N-n_1} \Big) \frac{1}{n_2 -1} \sum_{i\in s_2} \{ e_i(s_1) -  \bar{e}_{s_2}(s_1) \}^2 \\
\hat{V}_s\{ B(s_1) \mid s_1\} = \sum_{i\in s_2} \sum_{j\in s_2} \Big( 1 - \frac{\pi_{2i} \pi_{2j}}{\pi_{2ij}} \Big) e_i(s_1) e_j(s_1)
\end{gather*}
given $\pi_{2i} \equiv \frac{n_2}{N-n_1}$, and $\pi_{2ij} = \pi_{2i}$ if $i=j$ or $\pi_{2ij} =\frac{n_2(n_2-1)}{(N-n_1)(N-n_1-1)}$ if $i\neq j$. A design-unbiased estimator of the MSE of the corresponding $\hat{P}$ follows as
\begin{equation} \label{mseSRBkNN}
\overline{\mbox{mse}} = E_q\big( \mbox{mse}(s_1) \mid s \big) - V_q\big( \hat{P}(s_1) \mid s \big) .
\end{equation}

\paragraph{MSE of residual-tuned SRB-kNN} The error $B_c(s_1)$ is given earlier. Similarly to that of SRB-kNN, we have
\[
\mbox{mse}_c(s_1) = \hat{B}_c(s_1)^2 - \hat{V}_s\{ \hat{B}_c(s_1) \mid s_1\} + \hat{V}_s\{ B_c(s_1) \mid s_1\}
\]
where 
\[
\hat{V}_s\{ \hat{B}_c(s_1) \mid s_1\} = n_1^2 \Big( \frac{1}{n} - \frac{1}{N} \Big)^2 v_2 \qquad\text{and}\qquad
\hat{V}_s\{ B_c(s_1) \mid s_1\} = \Big( \frac{n_2}{n}\Big)^2v_2 .
\]
Notice that, regarding the variance of $B_c(s_1)$ conditional on $s_1$, we can simplify the random terms in $B_c(s_1)$ as 
\begin{align*}
\frac{N-n}{N} \bar{e}_{U\setminus s}(s_1) - \frac{N-n}{N} \frac{n_2}{n} \bar{e}_{s_2}(s_1)
& = \frac{N-n_1}{N} \bar{e}_{U\setminus s_1}(s_1) - \frac{n_2}{N} \bar{e}_{s_2}(s_1) - \frac{N-n}{n} \frac{n_2}{N} \bar{e}_{s_2}(s_1) \\
& = \frac{N-n_1}{N} \bar{e}_{U\setminus s_1}(s_1) - \frac{n_2}{n} \bar{e}_{s_2}(s_1)
\end{align*}
where $\bar{e}_{U\setminus s_1}(s_1)$ is a constant given $s_1$. A design-unbiased estimator of the MSE of the corresponding $\tilde{P}$ follows as
\begin{equation} \label{mseResidTuned}
\overline{\mbox{mse}}_c = E_q\big( \mbox{mse}_c(s_1) \mid s \big) - V_q\big( \hat{P}(s_1) \mid s \big) .
\end{equation}

\paragraph{MSE of OOB-tuned SRB-kNN} The error due to the tuned $\eta(x, s_1)$ is given as
\[
B_{\tau}(s_1) = \frac{1}{N} \sum_{i\notin s} e_i(s_1) - \frac{N-n}{N} \frac{1}{n_2} \sum_{i\in s_2} e_i(s_1) 
= \frac{N-n}{N} \big( \bar{e}_{U\setminus s}(s_1) - \bar{e}_{s_2}(s_1) \big)
\]
which has expectation $0$ conditional on $s_1$, since tuning has made it unbiased. It follows that we can use $\hat{B}_{\tau}(s_1) \equiv 0$, such that
$V_s\{ \hat{B}_{\tau}(s_1) \mid s_1\} \equiv 0$ as well. Regarding the variance of $B_{\tau}(s_1)$ conditional on $s_1$, we have now
\begin{align*}
\bar{e}_{U\setminus s}(s_1) - \bar{e}_{s_2}(s_1) 
& = \frac{N-n_1}{N-n} \bar{e}_{U\setminus s_1}(s_1) - \frac{n_2}{N-n} \bar{e}_{s_2}(s_1) - \bar{e}_{s_2}(s_1) \\
& = \frac{N-n_1}{N-n} \bar{e}_{U\setminus s_1}(s_1) - \frac{N-n_1}{N-n} \bar{e}_{s_2}(s_1)
\end{align*}
where $\bar{e}_{U\setminus s_1}(s_1)$ is a constant given $s_1$, such that 
\[
\mbox{mse}_{\tau}(s_1) = \hat{V}_s\{ B_{\tau}(s_1) \mid s_1\} = \Big( 1 - \frac{n_1}{N} \Big)^2 v_2 .
\]
A design-unbiased estimator of the MSE of the corresponding $\hat{P}$ follows as
\begin{equation} \label{mseOOBTuned}
\overline{\mbox{mse}}_{\tau} = E_q\big( \mbox{mse}_{\tau}(s_1) \mid s \big) - V_q\big( \hat{P}(s_1) \mid s \big) .
\end{equation}

\section{Monte Carlo SRB} \label{app:MC}

Exact calculation of $E_q(\cdot)$ or $V_q(\cdot)$ is infeasible if $\mathcal{C}(n, n_1)$ is too large. In such situations, the $q$-expectation and $q$-variance need to be evaluated by Monte Carlo based on $T$ (training, test) sets, denoted by $(s_1^{(t)}, s_2^{(t)})$ for $t =1, ..., T$.

For bias estimation by \eqref{biasSRBkNN}, one would calculate the $q$-expectation as 
\[
E_q\Big( \frac{1}{n_2} \sum_{i\in s_2} e_i(s_1)\Big) \approx \frac{1}{T} \sum_{t=1}^T \frac{1}{n_2} \sum_{i\in s_2^{(t)}} e_i(s_1^{(t)})
= \frac{1}{T} \sum_{t=1}^T \frac{1}{n_2} \sum_{i\in s_2^{(t)}} \big( \eta(x_i, s_1^{(t)}) - y_i\big) .
\]
Similarly for other $q$-expectations. 
For the $q$-variance in \eqref{mseSRBkNN}, one would use
\[
V_q\big( \hat{P}(s_1) \mid s \big) \approx \frac{1}{T-1} \Big( \sum_{t=1}^T \hat{P}(s_1^{(t)}) - \frac{1}{T} \sum_{t=1}^T \hat{P}(s_1^{(t)}) \Big)^2 .
\]
Similarly for the $q$-variance in \eqref{mseResidTuned} or \eqref{mseOOBTuned}.

Finally, as explained in Zhang et al. (2025, Section 2.3.2), one needs to estimate the MSE of the Monte Carlo SRB-predictor being implemented for out-of-sample prediction (instead of the exact SRB-predictor). Specifically, in terms of \eqref{mseSRBkNN} of SRB-kNN, one would use the unbiased estimator
\[
\tilde{\overline{\mbox{mse}}} = \frac{1}{T} \sum_{t=1}^T \Big( \mbox{mse}(s_1^{(t)}) 
- \Big\{ \hat{P}(s_1^{(t)}) - \frac{1}{T} \sum_{t=1}^T \hat{P}(s_1^{(t)}) \Big\}^2 \Big) .
\] 
Similarly for \eqref{mseResidTuned} of residual-tuned SRB-kNN, and \eqref{mseOOBTuned} of OOB-tuned SRB-kNN.


\begin{thebibliography}{999} 

\bibitem{angelopoulos}  Angelopoulos, A., Bates, S., Fannjiang, A., Jordan, M.I. and Zrnic, T. (2023). Prediction-powered inference. \textit{Science}, 382:669--674(2023). \url{DOI:10.1126/science.adi6000}

\bibitem{beaumont2022} Beaumont, J.-F. and Haziza, D. (2022). Statistical inference from finite population samples: A critical review of frequentist and Bayesian approaches. \textit{The Canadian Journal of Statistics}, 50:1186-1212.

\bibitem{blackwell1947} Blackwell, D. (1947). Conditional expectation and unbiased sequential estimation. \textit{Annals of Mathematical Statistics}, 18: 105-110.

\bibitem{breidt2017} Breidt, F. J. and Opsomer, J. D. (2017). Model-assisted survey estimation with modern prediction techniques. \textit{Statistical Science}, 32:190-205.

\bibitem{2001a} Breiman, L. (2001a).Statistical modelling: The two cultures. \textit{Statistical Science}, 16:199-231.

\bibitem{2001b} Breiman, L. (2001b). Random Forests. \textit{Machine Learning}, 45:5-32.

\bibitem{dagdoug2023} Dagdoug, M., Goga, C. and Haziza, D. (2021). Model-Assisted Estimation Through Random Forests in Finite Population Sampling. \textit{Journal of the American Statistical Association}, 118:1234-1251.

\bibitem{eurostat2017} European Commission(2017). \textit{European Statistics Code of Practice.} \url{https://ec.europa.eu/eurostat/web/quality/european-statistics-code-of-practice}

\bibitem{fisher1956}  Fisher, R.A. (1956). \textit{Statistical Methods and Scientific Inference}. Oliver and Boyd, Edinburgh and London.

\bibitem{hansen1987} Hansen, M. (1987). Some History and Reminiscences on Survey Sampling. \textit{Statistical Science}, 2:180-190.

\bibitem{horvitz1952survey} Horvitz, D. G. and Thompson, D. J. (1952). A generalization of sampling without replacement from a finite universe. \textit{Journal of the American Statistical Association}, 47:663-685.

\bibitem{kalton2002} Kalton, G. (2002). Models in practice of survey sampling. \textit{Journal of Official Statistics}, 18:129-154.

\bibitem{toth2019} McConville KS, and Toth D. (2019). Automated selection of post-strata using a model-assisted regression tree estimator. \textit{Scandinavian Journal of Statistics}, 46:389-413.

\bibitem{neyman1934} Neyman, J. (1934). On the Two Different Aspects of the Representative Method: The Method of Stratified Sampling and the Method of Purposive Selection. \textit{Journal of the Royal Statistical Society, pp. 558-625.}

\bibitem{rao1945} Rao, C. R. (1945). Information and accuracy attainable in the estimation of statistical parameters. \textit{Bulletin of Calcutta Mathematical Society}, 37:81-91.

\bibitem{rao2005} Rao, J. N. K. (2005). Interplay between sample survey theory and practice: An appraisal. \textit{Survey Methodology}, 31:117-138.

\bibitem{rao2011} Rao, J. N. K. (2011). Impact of frequentist and Bayesian methods on survey sampling practice: A selective appraisal. \textit{Statistical Science}, 26:240-256.

\bibitem{royall1970} Royall, R.M. (1970). On finite population sampling theory under certain linear regression models. \textit{Biometrika}, 57:377-387.

\bibitem{lcz2021} Sanguiao-Sande, L. and Zhang, L.-C. (2021). Design-Unbiased Statistical Learning in Survey Sampling. \textit{Sankhya A}, Centenary Issue in Honour of C. R. Rao, 83:714-744. 

\bibitem{statcan2017} Statistics Canada (2017). \textit{Quality Assurance Framework, 3rd edition.} \url{https://www150.statcan.gc.ca/n1/pub/12-539-x/12-539-x2019001-eng.htm}

\bibitem{sarndal1992} Särndal, C.-E., Swensson, B., and Wretman, J. (1992). \textit{Model-Assisted Survey Sampling}. Springer, New York.

\bibitem{smith1994} Smith, T. M. F. (1994). Sample surveys 1975–1990; an age of reconciliation? (with discussion). \textit{International Statistical Review}, 62:5-34.

\bibitem{ston1977} Stone, C.J. (1977). Consistent Nonparametric Regression. \textit{The Annals of Statistics}, 5:595-620.

\bibitem{UN2019} United Nations (2019). \textit{National Quality Assurance Frameworks Manual for Official Statistics.} \url{https://unstats.un.org/unsd/methodology/dataquality/}

\bibitem{valliant2000} Valliant, R., Dorfman, R. M., and Royall, R. M. (2000). \textit{Finite Population Sampling and Inference: A Prediction Approach}. Wiley, New York.

\bibitem{lcz2025} Zhang, L.-C., Sanguiao-Sande, L. and Lee, D. (2025). Design-based predictive inference. \textit{Journal of Official Statistics}, 41:404-432.


\end{thebibliography}
\end{document}